\documentclass[onefignum,onetabnum]{siamart251216}
\usepackage{amsmath,amssymb,mathtools,bm}
\usepackage{graphicx}
\usepackage{booktabs,multirow,array,tabularx,makecell}
\usepackage{enumitem}
\usepackage{xcolor}
\usepackage{algorithm}
\usepackage{algpseudocode}
\usepackage{placeins}
\usepackage{etoolbox}

\setlist[itemize]{leftmargin=.5in}
\setlist[enumerate]{leftmargin=.5in}

\newsiamremark{remark}{Remark}
\newsiamthm{assumption}{Assumption}

\newcommand{\R}{\mathbb{R}}
\newcommand{\E}{\mathbb{E}}
\newcommand{\KL}{\operatorname{KL}}
\newcommand{\TV}{\operatorname{TV}}
\newcommand{\Ent}{\operatorname{Ent}}

\newcommand{\tr}{\operatorname{tr}}

\newcommand{\dd}{\,\mathrm{d}}
\newcommand{\norm}[1]{\left\lVert #1\right\rVert}

\newcommand{\pibar}{\bar\pi}
\newcommand{\sbar}{\bar s}
\newcommand{\pr}{\mathrm{pr}}
\newcommand{\lik}{\mathrm{lik}}
\newcommand{\alg}{\mathrm{alg}}

\headers{A Posterior-Dynamics Framework}
{Z. Liu, T. Pang, R. Wang, and Y. Zheng}

\title{A Posterior-Dynamics Framework for Imaging Inverse Problems with Pretrained Diffusion Priors}

\author{Zhaoqiang Liu\thanks{School of Computer Science and Engineering, University of Electronic Science and Technology of China, Chengdu 611731, China (Zhaoqiang Liu: \email{zqliu12@gmail.com}; Ruibing Wang: \email{ruibingwang0818@163.com}; Yang Zheng: \email{202511081645@std.uestc.edu.cn}).}
\and Tongyao Pang\thanks{Corresponding author. Yau Mathematical Sciences Center, Shuangqing Complex, Tsinghua University, Beijing 100084, China (\email{typang@tsinghua.edu.cn}).}
\and Ruibing Wang\footnotemark[1]
\and Yang Zheng\footnotemark[1]}

\usepackage{float}

\makeatletter
\renewcommand\paragraph{\punct\@startsection{paragraph}{4}{.25in}%
  {\parskip}%
  {-.5em plus -.1em}%
  {\reset@font\normalsize\bfseries\itshape}}
\makeatother

\graphicspath{{assets/}}
\ifpdf
\hypersetup{
  pdftitle={A Posterior-Dynamics Framework for Imaging Inverse Problems with Pretrained Diffusion Priors},
  pdfauthor={Zhaoqiang Liu, Tongyao Pang, Ruibing Wang, Yang Zheng}
}
\fi

\begin{document}
\maketitle

\begin{abstract}
Pretrained diffusion models represent image distributions through a continuum
of progressively smoothed distributions. This multiscale structure organizes
generation from global structure to fine detail and supports high-quality,
diverse samples. We exploit the same multiscale diffusion prior for linear
imaging inverse problems. Rather than using the pretrained model only as a
denoiser in an outer iteration, we define a surrogate likelihood whose center is
aligned with the clean-image coordinate and whose covariance accounts for
residual diffusion uncertainty. This construction defines an explicit surrogate
posterior path, from which we derive continuous posterior dynamics.
A tunable Langevin component supports target tracking and allows the amount of
posterior exploration to be adapted to the application. We prove endpoint
consistency and a finite-horizon tracking bound and, in the exact-score setting,
first-order weak accuracy. For computation, we derive the Posterior-Dynamics
Implicit--Explicit sampler (PD-IMEX), a stable method using one score evaluation
per diffusion scale and an implicit data-consistency update. Experiments on
deblurring, super-resolution, and inpainting show strong reconstruction
quality at 100 score evaluations, coarse-grid stability, and controllable
fidelity--diversity behavior.
\end{abstract}



\section{Introduction}
Many imaging tasks, including deblurring, super-resolution, tomography, and accelerated magnetic resonance imaging, can be modeled as linear inverse problems. We consider measurements of the form
\begin{equation}\label{eq:measurement}
 y=Ax+n,
 \qquad A\in\R^{m\times d},
 \qquad n\sim\mathcal N(0,\sigma_d^2I),
\end{equation}
where $x\in\R^d$ is the unknown clean image. Because $A$ is often ill-conditioned or rank deficient, reconstruction depends critically on prior information. Classical models encode smoothness, sparsity, low-rank structure, or nonlocal self-similarity through explicit regularizers \cite{rudin1992,strong2003,chan2000,buades2005,dabov2007,ji2010,gu2014}. Task-specific restoration networks learn stronger regularities from paired, noisy, or self-supervised data \cite{zhang2017dncnn,liang2021swinir,zamir2022restormer,lehtinen2018noise2noise,batson2019noise2self,laine2019self}. Their reconstruction rule, however, is typically tied to the training degradation, data protocol, or task, and can require retraining or adaptation when the sensing operator changes.

Deep generative models provide a distributional alternative. They learn the
statistics of plausible images from clean data and can be reused with different
sensing operators at inference time. This separates the learned image prior from
the explicit physical measurement model and avoids training a new reconstruction
network for every degradation.

Diffusion and closely related continuous-time generative models have achieved
remarkable success in high-resolution image synthesis, producing realistic images
with coherent global structure, fine spatial detail, and broad semantic diversity.
Stable Diffusion \cite{rombach2022ldm} and the flow-matching model FLUX.1 Kontext
\cite{blackforestlabs2025flux} provide compelling empirical evidence that these
models can represent rich, high-dimensional image distributions. In a diffusion
model, generation is organized through a continuum
of progressively Gaussian-smoothed image distributions rather than through the
clean-image distribution alone \cite{song2019,ho2020,song2021}. We call this multi-level
noise-indexed family the \emph{multiscale diffusion prior}: each \emph{diffusion
scale} is a density at a different degree of Gaussian smoothing. At coarse
diffusion scales, the model organizes global structure and alternative modes; at
fine diffusion scales, it progressively resolves texture and local detail. This
coarse-to-fine organization underlies the high fidelity and diversity of
diffusion-based generation.

For inverse problems, this structure suggests conditioning the whole multiscale
diffusion prior rather than using the pretrained network only as an endpoint
denoiser. Coarse diffusion scales can organize global image structure while the
measurement remains uncertain, and finer scales can incorporate increasingly
detailed data consistency. Retaining a distribution at every diffusion scale also
allows multiple data-consistent modes in underdetermined problems
\cite{sun2024provable,wu2024plugplay}. Our objective is to exploit this multiscale
representation to obtain an efficient posterior solver without sacrificing the
diversity supplied by the generative prior.

Accordingly, we formulate reconstruction as Bayesian posterior inference and ask
how the endpoint posterior can be extended across the same diffusion scales:
\begin{equation}\label{eq:posterior}
 p(x\mid y)\propto p_0(x)\exp\!\left(-\frac{\norm{Ax-y}^2}{2\sigma_d^2}\right),
\end{equation}
where $p_0$ is represented implicitly by a diffusion model. A diffusion prior learns the score of progressively smoothed data distributions \cite{song2019,song2021,ho2020}. For a forward process
\begin{equation}\label{eq:general-forward}
 \dd X_t=f(X_t,t)\dd t+g(t)\dd W_t,
\end{equation}
the reverse-time prior dynamics are
\begin{equation}\label{eq:general-reverse}
 \dd X_t=\bigl[f(X_t,t)-g(t)^2\nabla\log p_t(X_t)\bigr]\dd t+g(t)\dd\bar W_t,
 \qquad t:T\to0,
\end{equation}
with the score approximated by a neural network trained by denoising score matching \cite{anderson1982,song2021}.

One principled way to modify the reverse prior dynamics so that their endpoint
law matches the posterior is to tilt the reference path law by the endpoint
likelihood. Let $\mathbb P^{\pr}$ denote the path law induced by
\eqref{eq:general-reverse}. The resulting likelihood-tilted path law admits the
relative-entropy-regularized variational characterization \cite{dupuis1997}
\begin{equation}\label{eq:path-control}
 \mathbb Q_y^\star
 =\operatorname*{arg\,min}_{\mathbb Q\ll\mathbb P^{\pr}}
 \left\{\KL(\mathbb Q\Vert\mathbb P^{\pr})
 -\E_{\mathbb Q}[\log p(y\mid X_0)]\right\},
 \qquad
 \frac{\dd\mathbb Q_y^\star}{\dd\mathbb P^{\pr}}
 =\frac{p(y\mid X_0)}{p(y)}.
\end{equation}
By Bayes' rule, the endpoint marginal under $\mathbb Q_y^\star$ is
$p_0(\cdot\mid y)$. This path-law modification can be realized dynamically by
a Doob transform \cite{doob1984} with potential
\[
 h_t(x)=\E_{\mathbb P^{\pr}}[p(y\mid X_0)\mid X_t=x]=p_t(y\mid x).
\]
Under the decreasing-time convention in \eqref{eq:general-reverse}, the Doob
transform contributes the drift $-g(t)^2\nabla\log h_t$. The exact posterior
dynamics are therefore
\begin{equation}\label{eq:exact-posterior-reverse}
 \dd X_t=\left[f(X_t,t)-g(t)^2\nabla\log p_t(X_t\mid y)\right]\dd t
 +g(t)\dd\bar W_t,
 \qquad t:T\to0.
\end{equation}
Bayes' rule then separates the posterior score in this controlled drift as
\begin{equation}\label{eq:posterior-score-decomp}
 \nabla_{x_t}\log p_t(x_t\mid y)
 =\nabla_{x_t}\log p_t(x_t)+\nabla_{x_t}\log p_t(y\mid x_t).
\end{equation}
The pretrained model supplies the prior score at every diffusion scale, but the
intermediate likelihood score is generally unavailable because
\begin{equation}\label{eq:intermediate-likelihood}
 p_t(y\mid x_t)=\int p(y\mid x_0)p(x_0\mid x_t)\dd x_0.
\end{equation}
Thus the central computational difficulty is to condition the multiscale
diffusion prior, not to learn another clean-image prior. The path-space control
formulation identifies the exact reference dynamics; the remaining task is to
construct a tractable surrogate for its unavailable likelihood term.

Existing methods address the unavailable intermediate likelihood through discrete
guidance corrections, operator-aware projections, unrolled or plug-and-play
reconstruction schemes, and more costly particle or inner Monte Carlo procedures
\cite{chung2022dps,chung2022mcg,song2023pigdm,kawar2022ddrm,zhu2023diffpir,wu2024plugplay,cardoso2023mcgdiff,wu2023asymptotic,zhang2025daps}.
Recent provable methods obtain posterior correctness through either
singular-direction DDIM updates or an outer probability transport equipped with
inner Langevin posterior-score estimation
\cite{jiao2026posteriorDDIM,chang2025provable}. These approaches have different
algorithmic structures. Our goal is to preserve the multiscale diffusion prior as
the central modeling object and derive an explicit posterior path, continuous
dynamics, and numerical solver from one construction.

We work in clean-image coordinates, where scale alignment and likelihood
continuation define a normalized surrogate posterior along the Gaussian smoothing
path of the diffusion prior. Its transport defect motivates a tunable Langevin
component for tracking and exploration. From the same continuous model, we derive
the Posterior-Dynamics Implicit--Explicit sampler (PD-IMEX), which uses one score
evaluation per scale and implicit linear data consistency.

Our main contribution is a multiscale posterior-dynamics framework that turns a
fixed pretrained diffusion model into an efficient and tunable solver for linear
imaging inverse problems. Specifically:
\begin{itemize}
 \item \textbf{Multiscale surrogate posterior construction.} We use the complete
 multiscale diffusion prior rather than treating the pretrained network only as a
 denoising module in an outer reconstruction procedure. Deterministic scale
 alignment places the surrogate likelihood center in the clean-image coordinate,
 while likelihood continuation accounts for the leading projected covariance of
 residual diffusion uncertainty. Together they define a normalized surrogate
 posterior at every diffusion scale and recover the desired Bayesian posterior at
 the endpoint.
 \item \textbf{Posterior dynamics with guarantees and tunable exploration.} We
 derive continuous dynamics consisting of base multiscale transport and a freely
 chosen Langevin component. For the moving surrogate posterior, the Langevin
 intensity supports target tracking and allows the same framework to range from
 nearly deterministic reconstruction to stronger posterior exploration. We
 characterize the path mismatch and establish endpoint consistency and a
 finite-horizon tracking bound.
 \item \textbf{An efficient and stable numerical realization.} We derive
 PD-IMEX, a one-score-evaluation-per-scale discretization that treats the linear
 likelihood dynamics implicitly and adds the Langevin innovation after the
 data-consistency solve. For the exact prior score, we establish first-order weak
 accuracy. Experiments show
 strong reconstruction quality at 100 score evaluations, coarse-grid stability,
 and control over the fidelity--diversity trade-off.
\end{itemize}

\section{Related work and positioning}
\paragraph{Likelihood guidance and operator-aware solvers}
Training-free diffusion inverse solvers commonly modify the unconditional
reverse process using an approximation of the unavailable intermediate
likelihood score at each diffusion scale. DPS evaluates measurement discrepancy
at a Tweedie denoised estimate. Methods based on manifold constraints or
pseudoinverses incorporate local geometry or approximate conditional covariance
information
\cite{chung2022dps,chung2022mcg,song2023pigdm}. TMPD goes beyond a point
approximation by using higher-order Tweedie moments to construct a statistically
motivated approximation of the conditional distribution \cite{boys2024tweedie}.
For structured linear problems, SNIPS and DDRM exploit singular-vector
decompositions to separate measurement-informed and prior-dominated directions,
DDNM decomposes range and null-space components, and DiffPIR alternates diffusion
denoising with explicit data consistency
\cite{kawar2021snips,kawar2022ddrm,wang2022ddnm,zhu2023diffpir}.

\paragraph{Posterior-oriented and annealed inference}
A second line of work places greater emphasis on posterior inference, typically
at additional computational cost. Variational approaches such as RED-Diff
optimize tractable objectives induced by the diffusion prior, while MGPS uses
intermediate-state guidance within a variational posterior-sampling construction
\cite{mardani2023reddiff,moufad2024mgps}. Filtering and sequential Monte Carlo
approaches, including auxiliary-variable methods, approximate the conditional
distribution through particles or inner Monte Carlo procedures rather than a
single reverse-guidance update
\cite{dou2024filtering,cardoso2023mcgdiff,wu2023asymptotic}. PnP-DM treats
the pretrained diffusion model as a plug-and-play prior and reduces Bayesian
inversion to iterative sampling of Gaussian denoising posteriors
\cite{wu2024plugplay}. DAPS instead constructs a decoupled noise-annealing process
and performs inner sampling at each diffusion scale, allowing larger moves and
improved exploration relative to a conventional reverse trajectory
\cite{zhang2025daps}. These approaches improve posterior fidelity or exploration
through additional inference iterations, whereas our focus is a single-loop
dynamics using one score evaluation at each discretized diffusion scale.

\paragraph{Provable diffusion posterior sampling}
Recent work has begun to establish rigorous posterior guarantees for diffusion
inverse solvers. For noisy linear inverse problems, a provable DDIM-type sampler
treats the singular directions of the measurement operator separately: each
component follows the diffusion prior while the measurement signal-to-noise ratio
is low and switches to a calibrated measurement predictor once the observation
becomes sufficiently informative. This gives lightweight coordinate-wise updates
with posterior consistency \cite{jiao2026posteriorDDIM}. For more general Bayesian
inverse problems, an outer probability transport estimates the time-dependent
posterior score through an inner Langevin Monte Carlo procedure with a warm start.
Its nonasymptotic analysis controls initialization and sampling errors in this
nested procedure, including the effects of problem conditioning
\cite{chang2025provable}. These results show that rigorous posterior
sampling can be obtained either by exploiting detailed structure of the
observation model or by investing additional computation in posterior-score
estimation.

\paragraph{Position of this work}
Our distinction lies in using the multiscale diffusion prior as the central object
from which posterior modeling and computation are derived. Many practical methods
introduce data consistency through discrete guidance, unrolled reconstruction, or
plug-and-play iterations, while recent provable methods rely on operator-specific
coordinate updates or inner Monte Carlo score estimation. Our construction is
complementary. We define a normalized surrogate posterior at every diffusion
scale, quantify its mismatch from the base multiscale transport, and derive
continuous transport--Langevin dynamics and the corresponding PD-IMEX
discretization from the same posterior path. This gives a single-loop,
one-score-per-scale solver
while retaining a tunable Langevin intensity for posterior tracking and
exploration. 

\section{Posterior dynamics from a multiscale diffusion prior}
The multiscale diffusion prior introduced in Section~1 becomes, in clean-image
coordinates, a Gaussian smoothing path. We first condition its endpoint exactly to
identify a reference posterior path and a family of transport--Langevin dynamics.
We then construct a tractable surrogate posterior path across the same diffusion
scales, analyze its tracking error, and derive the PD-IMEX discretization.

\subsection{Exact posterior path and transport--Langevin family}
We begin with the ideal setting in which the posterior score is available.
We identify the unknown signal with the diffusion endpoint $X_0\in\R^d$ and
assume throughout the total-variation analysis that $\sigma_d>0$. The
variance-preserving (VP) forward stochastic differential equation (SDE) is
\begin{equation}\label{eq:vp}
 \dd X_t=-\frac12\beta(t)X_t\dd t+\sqrt{\beta(t)}\dd W_t,
 \qquad X_0\sim p_0,
\end{equation}
where $\beta(t)>0$. Its marginal representation is
\begin{equation}\label{eq:vp-marginal}
 X_t=\sqrt{\alpha_t}X_0+\sigma_t\varepsilon,
 \qquad \alpha_t=\exp\!\left(-\int_0^t\beta(s)\dd s\right),
 \qquad \sigma_t^2=1-\alpha_t.
\end{equation}
The diffusion state contains both deterministic attenuation and additive Gaussian
noise. For the exact posterior path, it is convenient to remove only the known
attenuation and use the remaining noise-to-signal ratio as time. We therefore
introduce the clean state and log-SNR time
\begin{equation}\label{eq:clean-coordinate}
 Z_\lambda:=\frac{X_t}{\sqrt{\alpha_t}}
 =X_0+\tau_\lambda\varepsilon,
 \qquad \tau_\lambda=e^{-\lambda},
 \qquad q_\lambda:=\tau_\lambda^2=e^{-2\lambda},
 \qquad \lambda=\frac12\log\frac{\alpha_t}{1-\alpha_t}.
\end{equation}
This is an exact change of variables, not yet an approximation of the
likelihood. In this coordinate, every marginal is an additive Gaussian
perturbation of the same clean variable, with uncertainty $q_\lambda$. The
log-SNR clock makes the reverse direction monotone and gives
$\partial_\lambda q_\lambda=-2q_\lambda$. Thus increasing $\lambda$ follows the
reverse generative direction from coarse to fine diffusion scales. Let $P_q$ denote
convolution with $\mathcal N(0,qI)$. The unconditional clean-coordinate density
is
\begin{equation}\label{eq:pi-lambda}
 \pi_\lambda=P_{q_\lambda}p_0.
\end{equation}
The family $\{\pi_\lambda\}_\lambda$ is the clean-coordinate representation of
the multiscale diffusion prior.
Conditioning the initial law on $y$ gives the exact posterior path
\begin{equation}\label{eq:exact-posterior-path}
 \pi_\lambda^y=P_{q_\lambda}p_0(\cdot\mid y).
\end{equation}
Accordingly, $\{\pi_\lambda^y\}_\lambda$ is the exact posterior path across the
same diffusion scales.
For every finite $\lambda$, Gaussian smoothing makes $\pi_\lambda^y$ strictly positive and smooth, even when the endpoint posterior is only a Borel probability measure.

Under the change of variables in \eqref{eq:clean-coordinate}, the exact posterior
dynamics \eqref{eq:exact-posterior-reverse} become
\[
 \dd Z_\lambda
 =2q_\lambda\nabla\log\pi_\lambda^y(Z_\lambda)\dd\lambda
 +\sqrt{2q_\lambda}\dd W_\lambda.
\]
The marginal path \eqref{eq:exact-posterior-path}, however, does not uniquely
determine the pathwise dynamics. Let $a_\lambda\ge0$ denote the Langevin
intensity, which controls the strength of a matched Langevin score drift and
Brownian diffusion. The following proposition gives the resulting family with
the same marginals; the dynamics above are recovered when $a_\lambda=q_\lambda$.

\begin{proposition}[Exact posterior family with arbitrary Langevin intensity]\label{prop:exact-family}
Let $a_\lambda\ge0$ be locally bounded and suppose the SDE and its Fokker--Planck equation are well posed on every compact interval $[\lambda_0,\Lambda]$. Consider
\begin{equation}\label{eq:exact-posterior-sde}
 \dd Z_\lambda
 =(q_\lambda+a_\lambda)\nabla\log\pi_\lambda^y(Z_\lambda)\dd\lambda
 +\sqrt{2a_\lambda}\dd W_\lambda,
 \qquad Z_{\lambda_0}\sim\pi_{\lambda_0}^y.
\end{equation}
Then $Z_\lambda\sim\pi_\lambda^y$ for all $\lambda\ge\lambda_0$. Moreover, $\pi_\lambda^y\Rightarrow p_0(\cdot\mid y)$ as $\lambda\to\infty$.
\end{proposition}

The proof is given in \cref{app:exact}. The key identity is
\begin{equation}\label{eq:backward-heat}
 \partial_\lambda\pi_\lambda^y=-q_\lambda\Delta\pi_\lambda^y.
\end{equation}
When the candidate density in the Fokker--Planck equation of \eqref{eq:exact-posterior-sde} is set equal to $\pi_\lambda^y$, the $a_\lambda$-dependent drift and diffusion cancel exactly, leaving \eqref{eq:backward-heat}.

\begin{remark}[Relative Langevin intensity]\label{rem:zeta}
Choosing $a_\lambda=\zeta q_\lambda$, with $\zeta\ge0$, parameterizes the
Langevin component relative to the multiscale transport coefficient $q_\lambda$.
It gives
\[
 \dd Z_\lambda=(1+\zeta)q_\lambda\nabla\log\pi_\lambda^y(Z_\lambda)\dd\lambda
 +\sqrt{2\zeta q_\lambda}\dd W_\lambda.
\]
Thus $\zeta=0$ is probability-flow transport and $\zeta=1$ is the usual reverse
SDE in the clean coordinate. Values $\zeta>1$ are also admissible, and a general
$a_\lambda$ permits a time-dependent Langevin-intensity schedule. In the ideal
model, varying $\zeta$ changes pathwise exploration without changing the
prescribed posterior marginals. For the moving surrogate posterior introduced
below, the same Langevin component additionally supplies entropy dissipation and
therefore acts as a tracking mechanism.
\end{remark}

\subsection{Surrogate posterior path across diffusion scales}
\paragraph{Scale-consistent likelihood center}
The exact family is not directly implementable because an unconditional diffusion
model supplies the prior score but not the intermediate likelihood in
\eqref{eq:intermediate-likelihood}. Write the clean likelihood, up to its
normalizing constant, as
\[
 \ell(x_0):=\exp\!\left[-\frac{\norm{Ax_0-y}^2}{2\sigma_d^2}\right].
\]
The exact quantity is the conditional average
\[
 p_t(y\mid x_t)\propto\E[\ell(X_0)\mid X_t=x_t],
 \qquad
 X_0=\frac{x_t}{\sqrt{\alpha_t}}
 -\frac{\sigma_t}{\sqrt{\alpha_t}}\varepsilon.
\]
This decomposition separates two effects of the diffusion scale on the
intermediate likelihood. The deterministic attenuation determines where the
clean-image forward model should be evaluated, while the additive term represents
residual uncertainty about the clean image. Evaluating the latter exactly requires
the conditional law of the additive correction, including its state-dependent
mean, covariance, and higher moments. We instead approximate the two effects in
stages. First, we retain only the deterministic scaling of the diffusion channel
and drop the additive correction:
\[
 X_0\approx\frac{x_t}{\sqrt{\alpha_t}},
 \qquad
 p_t(y\mid x_t)\approx
 \ell\!\left(\frac{x_t}{\sqrt{\alpha_t}}\right).
\]
Here scale consistency refers specifically to the likelihood center: the known
attenuation is removed before applying the clean-image forward model. This
zeroth-order approximation does not yet account for uncertainty in the discarded
additive term. It gives
\begin{equation}\label{eq:scale-consistent-likelihood}
 L_t^{\mathrm{sc}}(x_t)
 :=\exp\!\left[-\frac{1}{2\sigma_d^2}
 \norm{A\frac{x_t}{\sqrt{\alpha_t}}-y}^2\right].
\end{equation}
The chain rule gives
\begin{equation}\label{eq:scale-consistent-score}
 \nabla_{x_t}\log L_t^{\mathrm{sc}}(x_t)
 =-\frac{1}{\sqrt{\alpha_t}\sigma_d^2}
 A^\top\!\left(A\frac{x_t}{\sqrt{\alpha_t}}-y\right).
\end{equation}
The approximation becomes the clean likelihood at the data endpoint, but at a
finite diffusion scale it remains only a surrogate for the conditional average.

In the $z$ coordinate, the baseline surrogate posterior is
\begin{equation}\label{eq:baseline-target}
 \widetilde\pi_\lambda^{\mathrm{sc}}(z\mid y)
 \propto \pi_\lambda(z)
 \exp\!\left[-\frac{\norm{Az-y}^2}{2\sigma_d^2}\right].
\end{equation}
Equation~\eqref{eq:baseline-target} reveals a simple prior-smoothing
structure of the clean-coordinate formulation. The physical likelihood is
the same as that at the data endpoint, while all diffusion-scale dependence
is confined to
\[
\pi_\lambda=P_{q_\lambda}p_0,
\]
the Gaussian-smoothed prior density.  Related prior-side
smoothing ideas appear in proximal optimization and proximal MCMC, where a
possibly nonsmooth prior energy is regularized through its Moreau--Yosida
envelope and the associated proximity operator
\cite{parikhBoyd2014,pereyra2016proximal}.
The analogy is structural rather than exact: proximal methods use Moreau
smoothing of the prior energy, whereas diffusion applies Gaussian smoothing
to the prior density. The $z$-coordinate makes this separation explicit by
removing the deterministic attenuation of the diffusion state and expressing
the baseline surrogate as a fixed physical likelihood multiplied by a
multiscale, Gaussian-smoothed prior.
\paragraph{Likelihood continuation for residual uncertainty}
The baseline surrogate in~\eqref{eq:baseline-target} therefore already
contains a prior-side smoothing path. The likelihood continuation introduced
next serves a different purpose: it accounts for the residual uncertainty of
the clean image conditional on the current diffusion state. Away from the
endpoint,
\[
Z_\lambda=X_0+\tau_\lambda\varepsilon
\]
remains uncertain around $X_0$, so evaluating the endpoint likelihood
directly at $Z_\lambda$ can be overconfident. We therefore retain a leading
covariance approximation of the discarded diffusion perturbation. This motivates the
effective covariance
\begin{equation}\label{eq:Gamma}
 \Gamma_\lambda:=\sigma_d^2I+\eta_\lambda AA^\top,
 \qquad \eta_\lambda\ge0,
 \qquad \eta_\lambda\to0\quad(\lambda\to\infty).
\end{equation}
We call $\eta_\lambda$ the likelihood-continuation schedule. The nominal choice
$\eta_\lambda=q_\lambda$ matches the covariance of the projected isotropic
perturbation. This covariance correction remains a controlled surrogate rather
than the exact state-dependent conditional likelihood. More generally, the
likelihood continuation accounts for the diffusion-scale uncertainty remaining in
$Z_\lambda$ and introduces measurement information progressively as the prior
moves from coarse to fine diffusion scales. The original likelihood is recovered
as $\eta_\lambda\to0$. Define
\begin{equation}\label{eq:Llambda}
 L_\lambda(z)
 :=\exp\!\left[-\frac12(Az-y)^\top\Gamma_\lambda^{-1}(Az-y)\right]
\end{equation}
and the normalized surrogate posterior at diffusion scale $\lambda$
\begin{equation}\label{eq:target}
 \pibar_\lambda(z\mid y)
 :=\frac{1}{\mathcal Z_\lambda}\pi_\lambda(z)L_\lambda(z),
 \qquad \mathcal Z_\lambda:=\int\pi_\lambda(z)L_\lambda(z)\dd z.
\end{equation}
Its score is
\begin{equation}\label{eq:target-score}
 \sbar_\lambda(z;y)
 =s_\lambda^{\pr}(z)+s_\lambda^{\lik}(z)
 =\nabla\log\pi_\lambda(z)-H_\lambda z+b_\lambda,
\end{equation}
where
\begin{equation}\label{eq:Hb}
 H_\lambda:=A^\top\Gamma_\lambda^{-1}A,
 \qquad b_\lambda:=A^\top\Gamma_\lambda^{-1}y.
\end{equation}
Let $t=t(\lambda)$ denote the inverse log-SNR parametrization and write
$\alpha_\lambda:=\alpha_{t(\lambda)}$. We use
$\widehat\varepsilon_\theta(\cdot,\lambda)$ as shorthand for the pretrained
network evaluated at the corresponding diffusion-time label
$t(\lambda)$. In practice,
\begin{equation}\label{eq:score-network}
 s_{\theta,\lambda}^{\pr}(z)
 :=-e^\lambda\widehat\varepsilon_\theta(\sqrt{\alpha_\lambda}z,\lambda)
\end{equation}
approximates $\nabla\log\pi_\lambda(z)$.

The likelihood continuation is useful only if it returns to the desired Bayesian model
as the diffusion noise vanishes. The next proposition verifies this endpoint
property separately for densities and scores. The distinction is important:
density convergence identifies the limiting law, whereas score convergence alone
does not guarantee convergence of normalized probability measures.

\begin{proposition}[Endpoint and score consistency]\label{prop:endpoint}
Assume that $p_0$ is absolutely continuous with an $L^1(\R^d)$ probability density and $\sigma_d>0$. If $\eta_\lambda\to0$, then
\begin{equation}\label{eq:endpoint-TV}
 \norm{\pibar_\lambda(\cdot\mid y)-p(\cdot\mid y)}_{\TV}\to0.
\end{equation}
If, in addition, $p_0$ is strictly positive and $C^1$, and
\[
 \nabla\log(P_{q_\lambda}p_0)\to\nabla\log p_0
 \quad\text{locally uniformly},
\]
then $\sbar_\lambda(\cdot;y)$ converges locally uniformly to the clean-posterior score
\[
 \nabla\log p_0(z)-\sigma_d^{-2}A^\top(Az-y).
\]
\end{proposition}

The $L^1$ density assumption is essential for total-variation convergence. If the endpoint prior is singular, Gaussian mollifications still converge weakly but need not converge in total variation. The proof of \cref{prop:endpoint} is in \cref{app:endpoint}. Thus the surrogate posterior path has the correct endpoint, but this statement alone says nothing about whether a dynamics can track it.

\paragraph{Transport defect and posterior dynamics}
The remaining issue comes from the way the surrogate was constructed. The exact
posterior path in \eqref{eq:exact-posterior-path} is a Gaussian smoothing of the
joint endpoint density and therefore satisfies the backward heat equation. By
contrast, \eqref{eq:target} multiplies a smoothed prior by a separately continued
likelihood. This product has the correct endpoint but need not follow the same
transport equation. Consequently, simply substituting $\sbar_\lambda$ into the
probability-flow drift does not in general produce marginals $\pibar_\lambda$.
We quantify this continuous-time mismatch before deciding how much correction is
needed.

\begin{definition}[Transport defect]\label{def:defect}
For a positive smooth surrogate posterior path $\{\pibar_\lambda\}$, define
\begin{equation}\label{eq:defect}
 r_\lambda(z)
 :=\frac{\partial_\lambda\pibar_\lambda(z)+q_\lambda\Delta\pibar_\lambda(z)}{\pibar_\lambda(z)}
 =\partial_\lambda\log\pibar_\lambda(z)
 +q_\lambda\Bigl(\Delta\log\pibar_\lambda(z)+\norm{\sbar_\lambda(z)}^2\Bigr).
\end{equation}
\end{definition}

The transport defect measures the mismatch between the surrogate posterior path
and the base multiscale transport inherited from the diffusion prior. Indeed, the
base transport velocity is
\[
 v_\lambda(z)=q_\lambda\nabla\log\pibar_\lambda(z),
\]
and \eqref{eq:defect} is equivalently
\begin{equation}\label{eq:defect-continuity}
 r_\lambda\pibar_\lambda
 =\partial_\lambda\pibar_\lambda
 +\nabla\cdot(\pibar_\lambda v_\lambda).
\end{equation}
Thus $r_\lambda$ is a normalized Fokker--Planck or continuity-equation residual. It is a signed local source term rather than a nonnegative norm. Under sufficient decay,
\begin{equation}\label{eq:defect-centered}
 \int r_\lambda\dd\pibar_\lambda=0.
\end{equation}
Positive and negative values describe where the prescribed surrogate posterior
gains or loses mass faster than the base multiscale transport can provide.

For an observable $F$, define its centered version by
$\mathcal C_{\pibar_\lambda}[F]
:=F-\E_{\pibar_\lambda}[F]$. The following identity applies directly to both
the nominal and practical continuation schedules.

\begin{proposition}[Transport defect for a differentiable likelihood continuation]\label{prop:explicit-defect}
Assume that $\eta_\lambda$ is differentiable and that the required differentiation and integration-by-parts operations are justified. Then
\begin{equation}\label{eq:explicit-defect}
 \begin{aligned}
 r_\lambda
 ={}&2q_\lambda\,
 \mathcal C_{\pibar_\lambda}
 \!\left[s_\lambda^{\pr}\!\cdot s_\lambda^{\lik}\right]\\
 &+\left(q_\lambda+\frac12\dot\eta_\lambda\right)
 \mathcal C_{\pibar_\lambda}
 \!\left[\norm{s_\lambda^{\lik}}^2\right].
 \end{aligned}
\end{equation}
In particular, $\eta_\lambda=q_\lambda$ implies
$\dot\eta_\lambda=-2q_\lambda$, so the second term vanishes.
\end{proposition}

The proof is in \cref{app:defect}. The first term is the prior--likelihood
interaction created because the heat flow does not preserve products. The
second records the mismatch between the likelihood continuation clock and the
nominal variance $q_\lambda$. For the practical schedule
$\eta_\lambda=q_\lambda^2/(q_\lambda+\tau_c^2)$,
\begin{equation}\label{eq:practical-defect-coefficient}
 q_\lambda+\frac12\dot\eta_\lambda
 =\frac{q_\lambda\tau_c^4}{(q_\lambda+\tau_c^2)^2},
\end{equation}
so continuation adds an explicit likelihood self-interaction to the nominal
prior--likelihood defect.

The defect explains why the base multiscale transport alone can be fragile for the
surrogate posterior path. We therefore retain this transport and add a Langevin
component at every $\lambda$. When the diffusion scale is frozen, this component
preserves $\pibar_\lambda$; while the posterior path moves, it supplies the
dissipative mechanism in the tracking estimate and permits exploration of
posterior ambiguity. Whether this dissipation dominates the transport defect
depends on the
functional-inequality and defect-control conditions below. Let $a_\lambda\ge0$
denote the Langevin intensity. The continuous surrogate model is
\begin{equation}\label{eq:surrogate-sde}
 \boxed{
 \dd Z_\lambda
 =(q_\lambda+a_\lambda)\sbar_\lambda(Z_\lambda;y)\dd\lambda
 +\sqrt{2a_\lambda}\dd W_\lambda.}
\end{equation}
It may be decomposed algebraically as
\begin{equation}\label{eq:transport-mixing-decomp}
 \dd Z_\lambda
 =\underbrace{q_\lambda\sbar_\lambda(Z_\lambda;y)\dd\lambda}_{\text{base multiscale transport}}
 +\underbrace{\Bigl[a_\lambda\sbar_\lambda(Z_\lambda;y)\dd\lambda
 +\sqrt{2a_\lambda}\dd W_\lambda\Bigr]}_{\text{Langevin component}}.
\end{equation}
The second bracket has $\pibar_\lambda$ as its invariant law when $\lambda$ is
frozen. Thus the base multiscale transport advances the diffusion scale, while the
Langevin component supports tracking and exploration. Increasing $a_\lambda$
strengthens both effects, while the definition of the surrogate target
$\pibar_\lambda$ itself remains unchanged.

The continuous model is now fully specified. Its convergence requires two
distinct facts. First, the surrogate posterior must approach the clean posterior,
which was established in \cref{prop:endpoint}. Second, the sampler law must track
the moving surrogate posterior despite the defect $r_\lambda$. To address the latter, let
$\mu_\lambda$ denote the law of \eqref{eq:surrogate-sde} and define
\begin{equation}\label{eq:H}
 H(\lambda):=\KL(\mu_\lambda\Vert\pibar_\lambda),
 \qquad
 I(\mu_\lambda\Vert\pibar_\lambda)
 :=\int\norm{\nabla\log(\mu_\lambda/\pibar_\lambda)}^2\dd\mu_\lambda.
\end{equation}
The following theorem shows that the Langevin component produces
relative-Fisher-information dissipation, while the transport defect is the
forcing term generated by motion of the surrogate posterior path. Combining
this tracking estimate with endpoint consistency yields convergence to the
posterior under the stated conditions.

\begin{theorem}[Entropy identity and finite-horizon posterior tracking]\label{thm:continuous}
Assume that the coefficients, surrogate posterior path, and sampler densities are sufficiently regular for the SDE, Fokker--Planck equation, differentiation under the integral, and the integrations by parts below. Then
\begin{equation}\label{eq:entropy-identity}
 H'(\lambda)
 =-a_\lambda I(\mu_\lambda\Vert\pibar_\lambda)
 -\int r_\lambda\dd\mu_\lambda.
\end{equation}
Suppose further that, for each $\lambda$, $\pibar_\lambda$ satisfies the log-Sobolev inequality
\begin{equation}\label{eq:LSI}
 \Ent_{\pibar_\lambda}(f^2)
 \le \frac{2}{\rho_\lambda}\int\norm{\nabla f}^2\dd\pibar_\lambda,
 \qquad \rho_\lambda>0,
\end{equation}
and that for every law $\nu$ encountered in the evolution with finite relative entropy,
\begin{equation}\label{eq:defect-control}
 -\int r_\lambda\dd\nu
 \le c_\lambda\KL(\nu\Vert\pibar_\lambda)+\varepsilon_\lambda,
 \qquad c_\lambda,\varepsilon_\lambda\ge0.
\end{equation}
On a fixed interval $[\lambda_0,\Lambda]$, set
\begin{equation}\label{eq:kappa-entropy}
 k_\lambda:=2a_\lambda\rho_\lambda-c_\lambda.
\end{equation}
If $H(\lambda_0)<\infty$, then for every $\lambda\le\Lambda$,
\begin{equation}\label{eq:Gronwall}
 H(\lambda)
 \le e^{-\int_{\lambda_0}^{\lambda}k_u\dd u}H(\lambda_0)
 +\int_{\lambda_0}^{\lambda}
 e^{-\int_s^\lambda k_u\dd u}\varepsilon_s\dd s.
\end{equation}
\end{theorem}

The proof is in \cref{app:entropy}. The first term in \eqref{eq:Gronwall}
propagates the mismatch between the practical high-noise initialization and the
initial surrogate posterior; the second accumulates transport-defect forcing after
Langevin dissipation. The log-Sobolev and entropy-variational tools are standard
\cite{bakry2014,dupuis1997}. In particular, the entropy variational inequality
verifies \eqref{eq:defect-control} whenever, for some $\theta_\lambda>0$,
\begin{equation}\label{eq:variational-defect}
 c_\lambda=\theta_\lambda^{-1},\qquad
 \varepsilon_\lambda=\theta_\lambda^{-1}
 \log\int e^{-\theta_\lambda r_\lambda}\dd\pibar_\lambda<\infty.
\end{equation}
Combined with \cref{prop:endpoint}, the finite-horizon estimate yields
$\norm{\mu_\lambda-p(\cdot\mid y)}_{\TV}\to0$ whenever
$\int_{\lambda_0}^\infty k_u\dd u=\infty$ and the weighted defect term in
\eqref{eq:Gronwall} vanishes.

\subsection{PD-IMEX discretization}
PD-IMEX discretizes the surrogate posterior dynamics over a grid of diffusion
scales. At each diffusion scale, the nonlinear prior prediction is evaluated once,
the linear likelihood drift is treated implicitly, and the Langevin innovation is
added after the data-consistency solve.
Using the score decomposition \eqref{eq:target-score}, the continuous dynamics
\eqref{eq:surrogate-sde} can be written as
\begin{equation}\label{eq:expanded-surrogate-sde}
 \dd Z_\lambda
 =(q_\lambda+a_\lambda)
 \bigl[s_\lambda^{\pr}(Z_\lambda)-H_\lambda Z_\lambda+b_\lambda\bigr]
 \dd\lambda+\sqrt{2a_\lambda}\dd W_\lambda.
\end{equation}
The prior score is nonlinear and supplied by the network, while the likelihood
drift is linear with coefficients $H_\lambda,b_\lambda$ determined by the
likelihood-continuation covariance $\Gamma_\lambda$. A fully implicit treatment would
require repeated network evaluations, whereas an explicit likelihood step can
be unstable when $H_\lambda$ is stiff. This structure motivates an
implicit-explicit (IMEX) step
that evaluates the network once and treats the linear likelihood implicitly.

Let
\begin{equation}\label{eq:grid}
 \lambda_0<\lambda_1<\cdots<\lambda_K,
 \qquad h_k:=\lambda_{k+1}-\lambda_k,
\end{equation}
and write $q_k:=q_{\lambda_k}$ and $\tau_k:=\tau_{\lambda_k}$. By Tweedie's
identity, $\E[X_0\mid Z_\lambda=z]=z+q_\lambda s_\lambda^{\pr}(z)$. Thus the
noise-prediction network also gives the clean prediction
\begin{equation}\label{eq:x0hat}
 \widehat\varepsilon_k
 :=\widehat\varepsilon_\theta(\sqrt{\alpha_{\lambda_k}}z_k,\lambda_k),
 \qquad
 \widehat x_{0,k}:=z_k-\tau_k\widehat\varepsilon_k.
\end{equation}
Although score, noise, and clean predictions are algebraically equivalent at
$\lambda_k$, their frozen schemes have different numerical behavior. Direct
clean-data prediction remains effective in high-dimensional pixel spaces where
prediction of noised quantities degrades \cite{li2026back}. Moreover, recent error
analysis shows that posterior-mean freezing admits a discretization bound
governed by the metric-entropy dimension, whereas score freezing retains
ambient-dimension dependence \cite{pang2026forward}. We therefore freeze the
clean prediction, leaving the known Gaussian state dependence unfrozen:
\begin{equation}\label{eq:frozen-prior}
 s_\lambda^{\pr}(z)\approx\frac{\widehat x_{0,k}-z}{q_\lambda}.
\end{equation}
Set
$\chi_\lambda:=q_\lambda+a_\lambda$ and
$\kappa_\lambda:=\chi_\lambda/q_\lambda$. On
$[\lambda_k,\lambda_{k+1}]$, substituting \eqref{eq:frozen-prior} into the
dynamics \eqref{eq:expanded-surrogate-sde} gives
\begin{equation}\label{eq:frozen-step-sde}
 \dd Z_\lambda
 =\bigl[\kappa_\lambda(\widehat x_{0,k}-Z_\lambda)
 -\chi_\lambda H_\lambda Z_\lambda+\chi_\lambda b_\lambda\bigr]\dd\lambda
 +\sqrt{2a_\lambda}\dd W_\lambda.
\end{equation}
Equation \eqref{eq:frozen-step-sde} keeps the continuation dependence of
$H_\lambda,b_\lambda$ explicit through \eqref{eq:Gamma}--\eqref{eq:Hb}.

We integrate the mean-reverting prior part of \eqref{eq:frozen-step-sde}
exactly, freeze the likelihood coefficients at the right endpoint, and apply
backward Euler only to that linear likelihood term. The required step
coefficients are
\begin{align}
 \varrho_k&:=\exp\!\left(-\int_{\lambda_k}^{\lambda_{k+1}}\kappa_s\dd s\right),
 \label{eq:rho-step}\\
 \Omega_k&:=\int_{\lambda_k}^{\lambda_{k+1}}\chi_s\dd s,
 &V_k&:=2\int_{\lambda_k}^{\lambda_{k+1}}
 a_s\exp\!\left(-2\int_s^{\lambda_{k+1}}\kappa_u\dd u\right)\dd s.
 \label{eq:Omega-V}
\end{align}
Here $\varrho_k$ is the exact mean contraction of the frozen mean-reverting
component, $\Omega_k$ is the accumulated likelihood weight, and $V_k$ is the
conditional variance generated by the Langevin diffusion. With
$H_{k+1}:=H_{\lambda_{k+1}}$ and
$b_{k+1}:=b_{\lambda_{k+1}}$, the PD-IMEX step is
\begin{equation}\label{eq:pd-imex-step}
 \boxed{
 \begin{aligned}
 (I+\Omega_kH_{k+1})z_{k+1}^{\det}
 &=\widehat x_{0,k}+\varrho_k(z_k-\widehat x_{0,k})+\Omega_kb_{k+1},\\
 z_{k+1}&=z_{k+1}^{\det}+\sqrt{V_k}\,\xi_k,
 \qquad \xi_k\sim\mathcal N(0,I).
\end{aligned}}
\end{equation}

\begin{algorithm}[H]
\caption{PD-IMEX posterior sampler}\label{alg:pd-imex}
\begin{algorithmic}[1]
\Require Measurement $y$, operator $A$, measurement-noise level $\sigma_d$, score network $\widehat\varepsilon_\theta$, grid $\{\lambda_k\}_{k=0}^K$, likelihood-continuation schedule $\eta_\lambda$, Langevin-intensity schedule $a_\lambda$
\Ensure Approximate posterior sample $z_K$
\State Draw $x_{\lambda_0}\sim\mathcal N(0,I)$ and convert it to the clean coordinate: $z_0\gets x_{\lambda_0}/\sqrt{\alpha_{\lambda_0}}$
\For{$k=0,\ldots,K-1$}
  \State Form $\Gamma_{\lambda_{k+1}}$, $H_{k+1}$, and $b_{k+1}$ from \eqref{eq:Gamma} and \eqref{eq:Hb}
  \State Evaluate $\widehat\varepsilon_k$ and $\widehat x_{0,k}$ from \eqref{eq:x0hat}
  \State Compute $\varrho_k$, $\Omega_k$, and $V_k$ from \eqref{eq:rho-step}--\eqref{eq:Omega-V}
  \State Solve the deterministic IMEX equation in \eqref{eq:pd-imex-step} for $z_{k+1}^{\det}$
  \State Draw $\xi_k\sim\mathcal N(0,I)$ and set $z_{k+1}\gets z_{k+1}^{\det}+\sqrt{V_k}\xi_k$
\EndFor
\State Return $z_K$
\end{algorithmic}
\end{algorithm}

The implicit solve in \eqref{eq:pd-imex-step} stabilizes the deterministic
likelihood drift. The Gaussian term $\sqrt{V_k}\xi_k$ is the Langevin innovation
after integration through the frozen prior
mean-reverting subflow. Adding it after the likelihood solve prevents additional
filtering by the likelihood resolvent. Classical analyses of stiff stochastic
integrators identify the same loss of fluctuations when an under-resolved
relaxation is handled by an L-stable implicit update \cite{li2008implicit}. If the
innovation were instead placed inside the backward-Euler solve, a likelihood
eigenmode of curvature $\nu$ would retain only
$(1+\Omega_k\nu)^{-2}$ of its innovation variance. This post-solve ordering
preserves the innovation covariance prescribed by the update, while the following
proposition establishes first-order consistency with the continuous posterior
dynamics. It is not claimed to integrate the full stiff frozen covariance exactly.

\begin{proposition}[Post-solve Langevin innovation and local consistency]\label{prop:covariance}
Assume $H_{k+1}$ is symmetric positive semidefinite and let
$R_k=(I+\Omega_kH_{k+1})^{-1}$. Conditional on $z_k$, the post-solve update
\eqref{eq:pd-imex-step} has covariance $V_kI$. If the same Langevin
innovation is placed inside the implicit likelihood solve, its covariance is
$V_kR_k^2$; hence post-solve placement avoids artificial damping by the
likelihood resolvent.

If $a_\lambda$, $q_\lambda$, and the coefficients are $C^1$ on a fixed finite interval, then as $h_k\to0$,
\begin{equation}\label{eq:first-order-params}
 \varrho_k=1-\kappa_{\lambda_k}h_k+O(h_k^2),
 \qquad
 \Omega_k=\chi_{\lambda_k}h_k+O(h_k^2),
 \qquad
 V_k=2a_{\lambda_k}h_k+O(h_k^2).
\end{equation}
Consequently, when
$\widehat x_{0,k}=z_k+q_ks_{\lambda_k}^{\pr}(z_k)$, the conditional mean and
covariance of \eqref{eq:pd-imex-step} match those of
\eqref{eq:surrogate-sde} to first order.
\end{proposition}

The proof is in \cref{app:weak}. This statement concerns the innovation
covariance of the chosen update decomposition; it does not claim exact integration
of the full stiff frozen covariance. The first-order expansion fixes the finite
interval and its coefficients, so it is not uniform in a simultaneous stiff
limit where $\Omega_k\norm{H_{k+1}}$ remains order one under refinement.

Having established retention of the post-solve Langevin innovation and local
consistency, we next quantify the global weak error relative to the continuous
sampler in the exact-score setting. Let $Q_{k}$ denote
the exact transition of \eqref{eq:surrogate-sde} from $\lambda_k$ to
$\lambda_{k+1}$, and let $\widehat Q_k$ be the one-step kernel of
\eqref{eq:pd-imex-step} using the surrogate posterior score in
\eqref{eq:target-score}.

\begin{theorem}[First-order weak error of PD-IMEX]\label{thm:weak}
Fix $\Lambda>\lambda_0$ and let $\lambda_K=\Lambda$. Suppose PD-IMEX uses the
exact prior score, so that
$\widehat x_{0,k}=z_k+q_ks_{\lambda_k}^{\pr}(z_k)$, and write
$B_\lambda(z)=(q_\lambda+a_\lambda)\sbar_\lambda(z;y)$. Assume on
$[\lambda_0,\Lambda]$ that $B_\lambda$, $a_\lambda$, $H_\lambda$, and
$b_\lambda$ are $C^1$ in $\lambda$, and that the spatial derivatives of
$B_\lambda$ through order four have polynomial growth uniformly in $\lambda$.
Assume also that the exact
process and the numerical chains are nonexplosive and satisfy the polynomial
moment bounds used below, uniformly on the finite interval and over sufficiently
fine grids. For a test function $\phi\in C_b^4(\R^d)$, suppose that the
associated backward Kolmogorov solution has spatial derivatives through order
four with polynomial-growth bounds uniformly on $[\lambda_0,\Lambda]$.
Initialize both processes from the same law and let
\[
 h:=\max_{0\le k<K}h_k.
\]
Then there is a constant $C_{\phi,\Lambda}$, independent of the grid, such that
\begin{equation}\label{eq:weak-bound}
 \left|\E\phi(z_K^{\alg})-\int\phi\dd\mu_\Lambda\right|
 \le C_{\phi,\Lambda}h.
\end{equation}
\end{theorem}

These are finite-horizon weak-approximation assumptions. Stating the required
backward regularity explicitly avoids imposing stronger global conditions on the
endpoint prior than are needed for the argument.

The proof in \cref{app:weak} establishes the local expansion
\begin{equation}\label{eq:local-weak}
 \widehat Q_k f(z)
 =f(z)+h_k\mathcal L_{\lambda_k}f(z)+O\bigl((1+\norm{z}^m)h_k^2\bigr),
\end{equation}
where
\begin{equation}\label{eq:generator}
 \mathcal L_\lambda f
 =(q_\lambda+a_\lambda)\sbar_\lambda(\,\cdot\,;y)\cdot\nabla f
 +a_\lambda\Delta f.
\end{equation}
The exact transition has the same first-order expansion. A telescoping argument with backward Kolmogorov regularity then gives the global result.

Combining the finite-horizon weak error with continuous tracking and endpoint
consistency yields the following decomposition for the oracle-score chain.

\begin{corollary}[Finite-horizon weak-error decomposition for oracle-score PD-IMEX]\label{cor:discrete}
Under the assumptions of \cref{thm:continuous,thm:weak}, for every $\phi\in C_b^4(\R^d)$,
\begin{equation}\label{eq:master-bound}
 \begin{aligned}
 \left|\E\phi(z_K^{\alg})-\int\phi\dd p(\cdot\mid y)\right|
 \le{}& C_{\phi,\lambda_K}h
 +2\norm{\phi}_\infty\sqrt{\frac12H(\lambda_K)}\\
 &+2\norm{\phi}_\infty
 \norm{\pibar_{\lambda_K}-p(\cdot\mid y)}_{\TV}.
 \end{aligned}
\end{equation}
\end{corollary}

The three terms in \eqref{eq:master-bound} are, respectively, time-discretization
error, sampler-to-surrogate tracking error, and surrogate endpoint error.

\section{Experiments}
The experiments evaluate the three layers of the framework: construction of the
surrogate posterior path from the multiscale diffusion prior, the role of the
Langevin intensity in tracking and exploration, and the numerical realization of
the dynamics.

\subsection{Experimental protocol}
We evaluate on 100 images from Flickr-Faces-HQ (FFHQ) $256\times256$ \cite{karras2019} and 100 images from ImageNet $256\times256$ \cite{russakovsky2015}. For FFHQ we use the pretrained prior from DPS \cite{chung2022dps}; for ImageNet we use the unconditional prior of Dhariwal and Nichol \cite{dhariwal2021}. We consider bicubic $\times4$ super-resolution, Gaussian deblurring with a $61\times61$ kernel of standard deviation $3.0$, and motion deblurring with a $61\times61$ kernel of intensity $0.5$. Independent Gaussian measurement noise with $\sigma_d=0.05$ is added in every main task. We report peak signal-to-noise ratio (PSNR), structural similarity index (SSIM), and learned perceptual image patch similarity (LPIPS) \cite{zhang2018lpips}.

We count each score-network call as one network function evaluation (NFE). The
reported 100-NFE outputs use 99 PD-IMEX transitions followed by one terminal
denoising evaluation, which serves as a finite-noise reconstruction readout.
The endpoints are the coarsest and finest trained diffusion scales,
$\lambda_0=-5.05884$ and $\lambda_K=4.60512$, and the grid is uniform in
log-SNR up to rounding to the nearest distinct training index. We parameterize
the Langevin intensity as $a_\lambda=\zeta q_\lambda$, where $\zeta$ is the
relative Langevin intensity. The likelihood covariance is
\begin{equation}\label{eq:experimental-continuation}
 \Gamma(\tau)=\sigma_d^2I+\eta(\tau)AA^\top,
 \qquad
 \eta(\tau)=\frac{\tau^4}{\tau^2+\tau_c^2},
 \qquad \tau_c=6.
\end{equation}
This choice follows the uncertainty interpretation in \eqref{eq:Gamma}. Since
$z=X_0+\tau\varepsilon$, the perturbation observed through $A$ has covariance
$\tau^2AA^\top$, giving the nominal likelihood covariance
$\sigma_d^2I+\tau^2AA^\top$. We temper it by the smooth gate
$\eta(\tau)/\tau^2=\tau^2/(\tau^2+\tau_c^2)$. Thus $\eta(\tau)\sim\tau^2$ at coarse diffusion scales,
where projected diffusion uncertainty should be retained, whereas
$\eta(\tau)\sim\tau^4/\tau_c^2$ near the endpoint, where the continuation is removed
more rapidly to expose the original likelihood. The crossover satisfies
$\eta(\tau_c)=\tau_c^2/2$. We fix $\tau_c=6$ for all main tasks; the resulting schedule
is positive, smooth, and endpoint-consistent.
We use $\zeta=0.1$ for FFHQ and $\zeta=0.01$ for ImageNet. MGPS and DAPS use multiple score calls within each outer level, so both levels and actual NFEs are reported.

Every likelihood solve is noniterative. Deblurring is diagonalized by the
two-dimensional FFT, bicubic super-resolution uses the explicit singular-vector
basis of the separable degradation, and inpainting uses a diagonal closed form.
The controlled ablations share the same images, measurement-noise realizations,
operator seed, and posterior-sampling seeds.

\FloatBarrier
\begin{table}[!ht]
\centering
\caption{Reconstruction results with $\sigma_d=0.05$, averaged over 100 images per dataset. Computational cost is reported as outer levels/actual score-network evaluations. Higher PSNR and SSIM and lower LPIPS are better.}
\label{tab:main}
\scriptsize
\setlength{\tabcolsep}{3.0pt}
\resizebox{\textwidth}{!}{%
\begin{tabular}{l c ccc ccc ccc}
\toprule
& Levels/ & \multicolumn{3}{c}{Gaussian deblur} & \multicolumn{3}{c}{Motion deblur} & \multicolumn{3}{c}{Super-resolution ($\times4$)}\\
Method & NFEs$\downarrow$ & PSNR$\uparrow$ & SSIM$\uparrow$ & LPIPS$\downarrow$ & PSNR$\uparrow$ & SSIM$\uparrow$ & LPIPS$\downarrow$ & PSNR$\uparrow$ & SSIM$\uparrow$ & LPIPS$\downarrow$\\
\midrule
\multicolumn{11}{l}{\textit{FFHQ}}\\
DPS \cite{chung2022dps} & 1000/1000 & 25.41 & .708 & \underline{.144} & 23.69 & .658 & .178 & 24.21 & .673 & .194\\
DiffPIR \cite{zhu2023diffpir} & 100/100 & 28.24 & .774 & .184 & 27.09 & .704 & .194 & 27.67 & .772 & \underline{.118}\\
DDRM \cite{kawar2022ddrm} & 100/100 & \underline{29.12} & \underline{.837} & \textbf{.132} & \underline{29.23} & \underline{.846} & \textbf{.089} & \underline{29.30} & \underline{.843} & .141\\
RED-Diff \cite{mardani2023reddiff} & 100/100 & 28.46 & .705 & .242 & 26.62 & .768 & .235 & 27.72 & .715 & .361\\
MGPS \cite{moufad2024mgps} & 100/814 & 27.88 & .792 & .149 & 26.82 & .769 & .147 & 27.55 & .783 & \textbf{.115}\\
DAPS \cite{zhang2025daps} & 100/500 & 28.71 & .792 & .155 & 28.39 & .795 & .142 & 28.66 & .764 & .142\\
Ours (PD-IMEX) & 100/100 & \textbf{29.52} & \textbf{.847} & .208 & \textbf{30.57} & \textbf{.869} & \underline{.139} & \textbf{29.49} & \textbf{.847} & .201\\
\midrule
\multicolumn{11}{l}{\textit{ImageNet}}\\
DPS \cite{chung2022dps} & 1000/1000 & 21.81 & .526 & .314 & 21.15 & .503 & .352 & 21.58 & .507 & .404\\
DiffPIR \cite{zhu2023diffpir} & 100/100 & 24.18 & .597 & .493 & 23.16 & .514 & .422 & 23.71 & .577 & .338\\
DDRM \cite{kawar2022ddrm} & 100/100 & \underline{25.59} & \underline{.714} & \textbf{.336} & \underline{26.61} & \underline{.761} & \textbf{.189} & \textbf{25.64} & \textbf{.720} & .321\\
RED-Diff \cite{mardani2023reddiff} & 100/100 & 24.52 & .661 & .491 & 23.76 & .635 & .429 & 24.46 & .609 & .537\\
MGPS \cite{moufad2024mgps} & 100/674 & 24.68 & .657 & .397 & 24.14 & .642 & .352 & 24.77 & .667 & \textbf{.305}\\
DAPS \cite{zhang2025daps} & 100/500 & 24.98 & .661 & \underline{.349} & 25.33 & .685 & .303 & 25.08 & .669 & \underline{.319}\\
Ours (PD-IMEX) & 100/100 & \textbf{25.62} & \textbf{.716} & .399 & \textbf{27.59} & \textbf{.779} & \underline{.197} & \underline{25.35} & \underline{.712} & .390\\
\bottomrule
\end{tabular}}
\end{table}

At 100 score-network evaluations, PD-IMEX provides the strongest overall
distortion performance in the comparison and remains competitive in perceptual
quality, with the expected fidelity--perception trade-off on several tasks. It
ranks first in ten of the twelve PSNR and SSIM comparisons and second in the
remaining two, including against methods using five to ten times as many score
evaluations. DDRM and MGPS retain lower LPIPS on several tasks, whereas PD-IMEX
is second in motion-deblurring LPIPS on both datasets.

\begin{figure}[t]
\centering
\includegraphics[width=\textwidth]{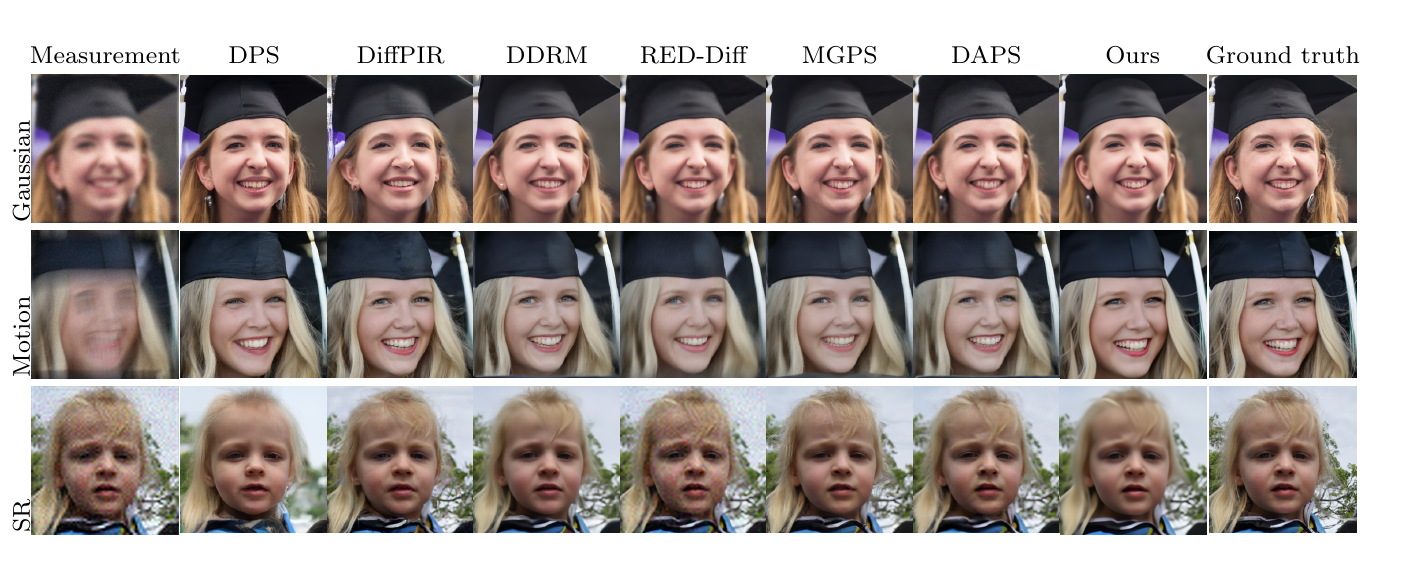}
\caption{Representative FFHQ reconstructions for the methods in \cref{tab:main}. Columns share the same measurement within each row.}
\label{fig:reconstructions}
\end{figure}

\subsection{Ablation of multiscale posterior construction}
We first test how the measurement model is coupled to the multiscale diffusion
prior. The naive baseline applies time-independent likelihood guidance, the
second configuration corrects the likelihood center through scale consistency,
and the full configuration additionally continues its covariance. All other
stochastic and numerical settings are fixed. The naive multiplier is tuned
separately for each task, giving $\gamma^\star=1$ on FFHQ and $\gamma^\star=0.3$
on ImageNet.

\begin{table}[t]
\centering
\caption{Ablation of the multiscale posterior construction over 100 images at a fixed 100-NFE budget. The naive baseline uses its best time-independent scalar multiplier. The second configuration aligns the likelihood center with the clean-image coordinate, and the third additionally continues its covariance to account for residual diffusion uncertainty. All Langevin and numerical settings are fixed.}
\label{tab:ablation}
\scriptsize
\setlength{\tabcolsep}{3.0pt}
\resizebox{\textwidth}{!}{%
\begin{tabular}{l ccc ccc}
\toprule
& \multicolumn{3}{c}{FFHQ Gaussian deblur} & \multicolumn{3}{c}{ImageNet motion deblur}\\
Configuration & PSNR$\uparrow$ & SSIM$\uparrow$ & LPIPS$\downarrow$ & PSNR$\uparrow$ & SSIM$\uparrow$ & LPIPS$\downarrow$\\
\midrule
Naive guidance (best scalar) & 19.228 & .67088 & .33503 & 23.799 & \underline{.70211} & .34586\\
Scale-consistent likelihood & \underline{29.502} & \textbf{.84730} & \textbf{.20597} & \underline{24.729} & .56936 & \underline{.20600}\\
Scale-consistent + continuation & \textbf{29.523} & \underline{.84723} & \underline{.20753} & \textbf{27.585} & \textbf{.77864} & \textbf{.19684}\\
\bottomrule
\end{tabular}}
\end{table}

\Cref{tab:ablation} separates the two components of the surrogate
likelihood. Scale consistency corrects its center by undoing the deterministic
attenuation of the diffusion state. A scalar guidance weight can compensate for
part of a center mismatch on some tasks but cannot replace this scale-dependent
alignment uniformly. Likelihood continuation then changes the effective
covariance, accounting for residual diffusion uncertainty at coarse scales. Its
effect is small on FFHQ Gaussian deblurring but substantial on ImageNet motion
deblurring.

\subsection{Langevin intensity: tracking and exploration}
\Cref{fig:exploration} shows that the practical effect of the relative Langevin intensity is
nonmonotone and task dependent. On FFHQ Gaussian deblurring, a moderate value
improves distortion accuracy, while a larger value can improve LPIPS at the
expense of PSNR. ImageNet motion deblurring favors a smaller value, and
an excessively large value degrades both metrics. Thus $\zeta$ controls a
finite-step tracking--exploration trade-off rather than serving as a universally
beneficial choice of Langevin intensity.

\begin{figure}[t]
\centering
\includegraphics[width=\textwidth]{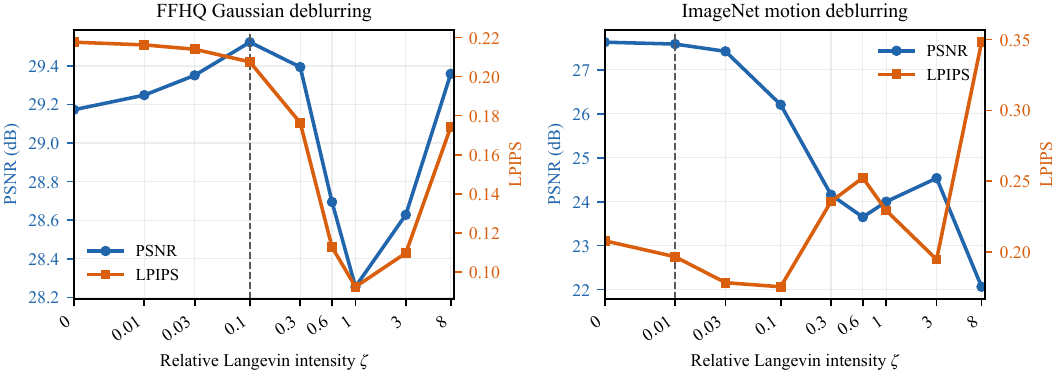}
\caption{Effect of the relative Langevin intensity $\zeta$ on reconstruction fidelity and perceptual quality for two noisy deblurring tasks. Each panel reports PSNR on the left axis and LPIPS on the right axis over 100 images. Only $\zeta$ is varied; the continuation, numerical grid, score network, measurements, and random seeds are fixed. Dashed vertical lines indicate the values used in \cref{tab:main}.}
\label{fig:exploration}
\end{figure}

\Cref{fig:exploration-samples} complements the aggregate metrics with a
multimodal inpainting example. At $\zeta=0$, the finite-step deterministic
implementation contracts different initializations toward similar, relatively
smooth completions. In contrast, the selected larger value produces coherent
but visibly different facial completions, illustrating how the relative
Langevin intensity controls sample variability in an underdetermined problem.

\begin{figure}[!ht]
\centering
\includegraphics[width=\textwidth]{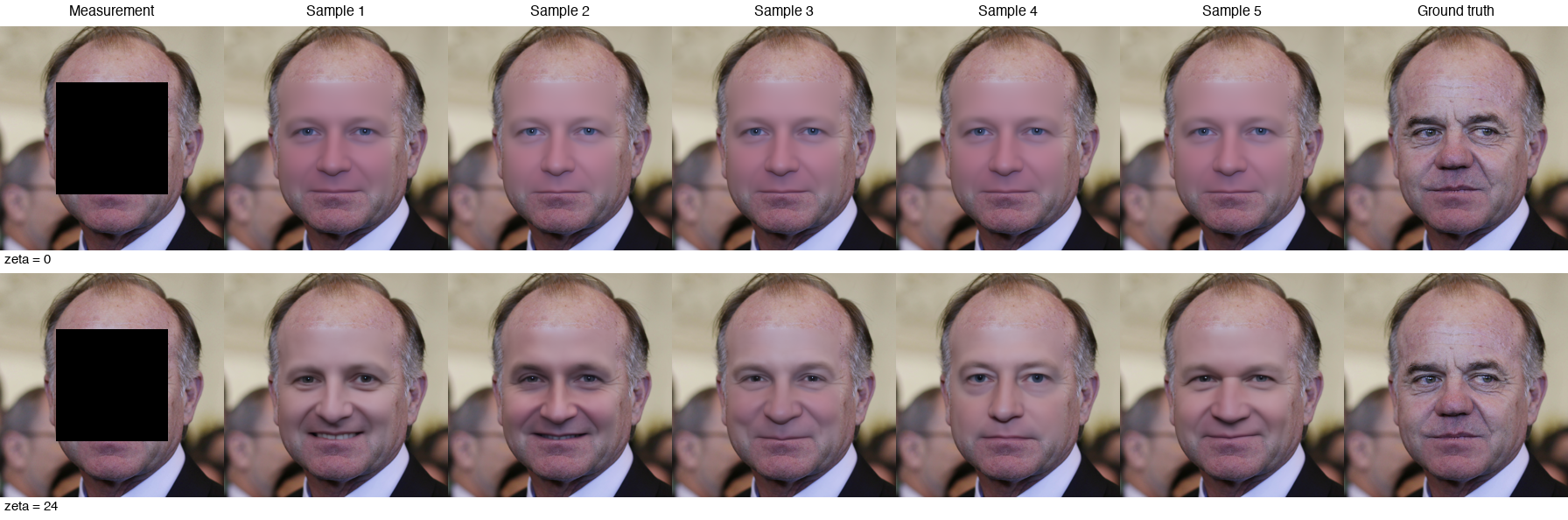}
\caption{Effect of the relative Langevin intensity on sample diversity for centered-box inpainting. The two rows use $\zeta=0$ and $\zeta=24$, respectively, and the five central columns use independent sampling seeds while keeping the measurement and operator fixed.}
\label{fig:exploration-samples}
\end{figure}

\subsection{Numerical realization: coarse-grid stability and innovation placement}
\paragraph{Coarse-grid stability}
We next examine the two numerical choices in PD-IMEX by refining the log-SNR
grid.  The explicit and implicit variants have the same surrogate posterior path and
Langevin-intensity schedule; they differ only in the treatment of the linear
likelihood drift.

\begin{table}[!t]
\centering
\caption{NFE refinement of the explicit and implicit likelihood updates. Entries are mean PSNR, SSIM, and LPIPS over 100 images. Within each NFE, the better result between the two updates is bold; \emph{fail} denotes numerical failure on all 100 images. All configurations use the same surrogate posterior path, Langevin-intensity schedule, score network, images, measurements, and random seeds; only the log-SNR grid and the treatment of the linear likelihood drift differ.}
\label{tab:nfe-refinement}
\scriptsize
\setlength{\tabcolsep}{3.6pt}
\resizebox{\textwidth}{!}{%
\begin{tabular}{lc ccc ccc}
\toprule
Task & NFE & \multicolumn{3}{c}{Explicit} & \multicolumn{3}{c}{Implicit}\\
\cmidrule(lr){3-5}\cmidrule(lr){6-8}
& & PSNR$\uparrow$ & SSIM$\uparrow$ & LPIPS$\downarrow$ & PSNR$\uparrow$ & SSIM$\uparrow$ & LPIPS$\downarrow$\\
\midrule
\multirow{4}{*}{\makecell[l]{FFHQ Gaussian\\deblur}} & 100 & \multicolumn{3}{c}{\emph{fail}} & \textbf{29.523} & \textbf{.8472} & \textbf{.2075}\\
 & 200 & 28.698 & .7891 & .2311 & \textbf{29.505} & \textbf{.8470} & \textbf{.2066}\\
 & 500 & \textbf{29.499} & \textbf{.8472} & \textbf{.2058} & 29.489 & .8469 & .2064\\
 & 1000 & \textbf{29.494} & \textbf{.8470} & \textbf{.2063} & 29.486 & .8468 & .2066\\
\midrule
\multirow{4}{*}{\makecell[l]{ImageNet motion\\deblur}} & 100 & 5.434 & .1344 & .9360 & \textbf{27.585} & \textbf{.7786} & \textbf{.1968}\\
 & 200 & 18.342 & .2843 & .5728 & \textbf{27.643} & \textbf{.7841} & \textbf{.2217}\\
 & 500 & \textbf{27.661} & \textbf{.7847} & \textbf{.2284} & 27.638 & .7838 & .2357\\
 & 1000 & \textbf{27.637} & \textbf{.7838} & \textbf{.2375} & 27.627 & .7833 & .2403\\
\bottomrule
\end{tabular}}
\end{table}

The explicit update is unusable on the 100-NFE grid: it fails numerically on
all FFHQ images and attains only $5.43$ dB PSNR, $.1344$ SSIM, and $.9360$
LPIPS on ImageNet. Refinement restores the same limiting behavior as the
implicit update, but at an operator-dependent rate. On FFHQ the explicit method
remains less accurate at 200 NFEs and agrees with the implicit method across all
three metrics by 500 NFEs; on ImageNet it remains under-resolved at 200 NFEs and
recovers only at 500 NFEs. Over the same refinement, the maximum stability
indicator $\max_k\Omega_k\lambda_{\max}(H_{k+1})$ decreases from $7.30$ to
$1.06$ at 500 NFEs and $.50$ at 1000 NFEs on FFHQ; the corresponding values
on ImageNet are $6.70$, $.97$, and $.46$.  The implicit update is
stable throughout. Thus implicit treatment follows the same posterior dynamics on
a substantially coarser grid and hence requires fewer score evaluations.

\paragraph{Innovation placement}
Noiseless centered-box inpainting is treated separately as a qualitative
stress test. Removing a $128\times128$ square leaves a large null space, so
many semantically different faces can agree exactly with the observed pixels.
Pixelwise PSNR and SSIM against one hidden ground truth would penalize plausible
posterior samples merely for choosing a different eye shape, expression, or
facial detail. We therefore use this task only to inspect visual coherence and
sample variability. All methods use the same FFHQ prior, centered mask, and
noiseless observed pixels. For PD-IMEX, we use a quadratic grid, initialize the
observed region from the appropriately noised measurement, and apply no hard
projection. The 100-NFE configuration uses $\tau_c=24$ and $\zeta=24$.

This task also makes the placement of the Langevin innovation especially
relevant. To display the distinction explicitly, write
\[
 d_k:=\widehat x_{0,k}+\varrho_k(z_k-\widehat x_{0,k})+\Omega_kb_{k+1},
 \qquad R_k:=(I+\Omega_kH_{k+1})^{-1}.
\]
Using the same Gaussian draw $\xi_k$, the two variants are
\begin{equation}\label{eq:innovation-placement}
 \begin{aligned}
 z_{k+1}^{\mathrm{inside}}
 &=R_k\bigl(d_k+\sqrt{V_k}\,\xi_k\bigr),\\
 z_{k+1}^{\mathrm{post}}
 &=R_kd_k+\sqrt{V_k}\,\xi_k,
 \qquad \xi_k\sim\mathcal N(0,I).
 \end{aligned}
\end{equation}
Thus, inside-solve placement filters the innovation through the likelihood
resolvent, whereas post-solve placement applies the resolvent only to the
deterministic update. Along a likelihood eigenmode of curvature $\nu$, the local
relaxation time is approximately $(\kappa_\lambda+\chi_\lambda\nu)^{-1}$.
Near the endpoint, the product $\Omega_k\nu$ can remain order one or larger, so
this fast scale is not resolved by the grid. The inside-solve variant then retains
only the fraction $(1+\Omega_k\nu)^{-2}$ of the innovation variance. The
post-solve variant in \eqref{eq:innovation-placement} instead stabilizes the
deterministic likelihood drift without filtering the Langevin innovation.

To isolate this placement effect at finite NFE, the amount of residual refresh
must also remain comparable.  Over step $k$, the fresh-innovation fraction is
$1-\exp(-2\zeta\Delta\lambda_k)$.  Thus, holding $\zeta$ fixed while refining
the grid weakens the refresh per step and confounds innovation placement with a
different practical stochastic regime.  We use
$\zeta=24,48,120,240$ at $100,200,500,1000$ NFEs, respectively, which
approximately preserves the typical per-step refresh strength of the 100-NFE
configuration.

Within each NFE, the inside- and post-solve variants share the initialization
and every Gaussian draw; only the placement of the innovation changes.
\Cref{fig:innovation-refinement} shows that post-solve innovation gives a
coherent completion throughout the tested NFE range, whereas filtering the
same innovation through the inside-solve resolvent leaves the missing region
unresolved. The experiment therefore supports post-solve placement when the
finite-step refresh strength is matched.

\begin{figure}[!ht]
\centering
\includegraphics[width=\textwidth]{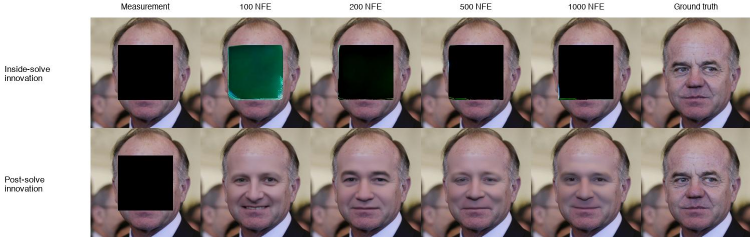}
\caption{Langevin innovation-placement study for noiseless centered-box inpainting. From 100 to 1000 NFEs, we use $\zeta=24,48,120,240$, respectively, to keep the typical per-step residual refresh comparable. Within each column, the inside- and post-solve variants use the same initialization, deterministic implicit update, and Gaussian draws; only the innovation placement changes.}
\label{fig:innovation-refinement}
\end{figure}

\Cref{fig:box-recon} places this controlled study in a broader inverse-problem
comparison.  Since many images are compatible with the same masked
observation, we do not rank the methods by pixelwise proximity to the single
hidden image.  Instead, the figure compares preservation of the observed
region and the generation of coherent, semantically plausible content in the
missing region.

\begin{figure}[H]
\centering
\includegraphics[width=.9\textwidth]{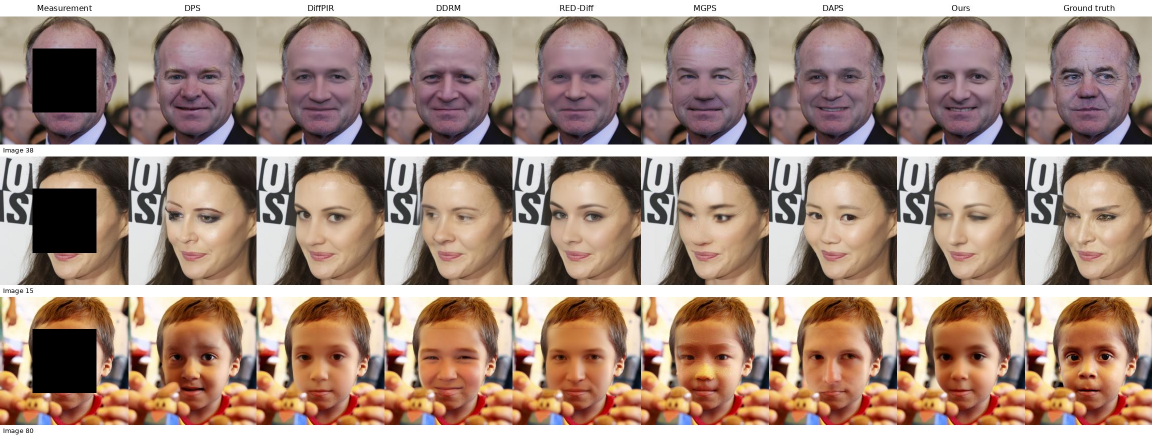}
\caption{Qualitative comparison for noiseless centered-box inpainting. From top to bottom, the rows show images 38, 15, and 80. All methods use the same FFHQ model, centered mask, and noiseless observed pixels. PD-IMEX uses $\zeta=24$, post-solve Langevin innovations, and 100 NFEs. Because the conditional problem is highly nonunique, the ground truth is shown as one compatible latent image rather than as the unique desired completion.}
\label{fig:box-recon}
\end{figure}
\FloatBarrier

\section{Conclusion}
We developed posterior dynamics that use the full multiscale diffusion prior for
linear imaging inverse problems. The key idea is to condition the diffusion prior
across diffusion scales, rather than using the pretrained network only as a
denoising module inside an outer reconstruction procedure. Aligning the likelihood
center with the clean-image coordinate and continuing its covariance for residual
diffusion uncertainty define an explicit surrogate posterior path. The resulting
dynamics combine base multiscale transport with a tunable Langevin component,
providing a common mechanism for posterior
tracking and task-dependent exploration. Endpoint consistency and a finite-horizon
tracking estimate give analytical support for this construction.

PD-IMEX provides an efficient and stable realization of the dynamics. It uses one
score evaluation per diffusion scale, treats the linear likelihood drift
implicitly, and places the Langevin innovation after the data-consistency solve.
Experiments show strong reconstruction quality at 100 score evaluations,
demonstrate that implicit integration preserves coarse-grid efficiency, and show
how the relative Langevin intensity adapts the fidelity--diversity behavior to
different inverse problems. Together, these results provide a direct route from a
pretrained multiscale diffusion prior to a theoretically supported and flexible
posterior solver.

\appendix
\section{Proof of the exact posterior family}\label{app:exact}
\begin{proof}[Proof of \cref{prop:exact-family}]
Let $q_\lambda=e^{-2\lambda}$. Since $\pi_\lambda^y=P_{q_\lambda}p_0(\cdot\mid y)$ and the Gaussian heat semigroup satisfies
\[
 \partial_qP_q\nu=\frac12\Delta P_q\nu,
\]
we have $q_\lambda'=-2q_\lambda$ and therefore
\begin{equation}\label{eq:app-heat}
 \partial_\lambda\pi_\lambda^y=-q_\lambda\Delta\pi_\lambda^y.
\end{equation}
Let $m_\lambda$ be the density of \eqref{eq:exact-posterior-sde}. Its Fokker--Planck equation is
\begin{equation}\label{eq:app-exact-fp}
 \partial_\lambda m_\lambda
 =-\nabla\cdot\left((q_\lambda+a_\lambda)m_\lambda\nabla\log\pi_\lambda^y\right)
 +a_\lambda\Delta m_\lambda.
\end{equation}
Substituting $m_\lambda=\pi_\lambda^y$ and using
$\pi_\lambda^y\nabla\log\pi_\lambda^y=\nabla\pi_\lambda^y$ gives
\[
 -\nabla\cdot\left((q_\lambda+a_\lambda)\nabla\pi_\lambda^y\right)
 +a_\lambda\Delta\pi_\lambda^y
 =-q_\lambda\Delta\pi_\lambda^y
 =\partial_\lambda\pi_\lambda^y.
\]
Thus $\pi_\lambda^y$ solves \eqref{eq:app-exact-fp}. Since the initial law is $\pi_{\lambda_0}^y$, uniqueness of the Fokker--Planck solution implies $m_\lambda=\pi_\lambda^y$ on every compact interval. Finally, $q_\lambda\to0$, and continuity of the Gaussian semigroup at zero gives
$P_{q_\lambda}p_0(\cdot\mid y)\Rightarrow p_0(\cdot\mid y)$.
\end{proof}

\section{Endpoint and score consistency}\label{app:endpoint}
\begin{proof}[Proof of \cref{prop:endpoint}]
Because $p_0\in L^1(\R^d)$ and the Gaussian kernels form an approximate identity,
\begin{equation}\label{eq:app-L1-prior}
 \norm{\pi_\lambda-p_0}_{L^1}\to0.
\end{equation}
Let
\[
 L_\infty(z)=\exp\!\left[-\frac{\norm{Az-y}^2}{2\sigma_d^2}\right].
\]
Since $\Gamma_\lambda^{-1}\to\sigma_d^{-2}I$, we have $L_\lambda(z)\to L_\infty(z)$ pointwise, and $0<L_\lambda,L_\infty\le1$. Write
$f_\lambda=\pi_\lambda L_\lambda$ and $f_\infty=p_0L_\infty$. Then
\begin{align}
 \norm{f_\lambda-f_\infty}_{L^1}
 &\le \norm{\pi_\lambda-p_0}_{L^1}
 +\int\pi_\lambda|L_\lambda-L_\infty|\dd z\notag\\
 &\le 2\norm{\pi_\lambda-p_0}_{L^1}
 +\int p_0|L_\lambda-L_\infty|\dd z\to0,\label{eq:app-unnormalized}
\end{align}
where the final integral converges by dominated convergence. Consequently,
\[
 \mathcal Z_\lambda=\int f_\lambda\dd z\to \mathcal Z_\infty:=\int f_\infty\dd z>0.
\]
For all sufficiently large $\lambda$, $\mathcal Z_\lambda\ge \mathcal Z_\infty/2$, and
\begin{align*}
 \norm{\frac{f_\lambda}{\mathcal Z_\lambda}-\frac{f_\infty}{\mathcal Z_\infty}}_{L^1}
 &\le \frac{1}{\mathcal Z_\lambda}\norm{f_\lambda-f_\infty}_{L^1}
 +\left|\frac1{\mathcal Z_\lambda}-\frac1{\mathcal Z_\infty}\right|\norm{f_\infty}_{L^1}
 \to0.
\end{align*}
Since total variation is one half of the $L^1$ distance between densities, this proves \eqref{eq:endpoint-TV}.

For the score statement, Bayes' rule gives
\[
 \nabla\log p(z\mid y)
 =\nabla\log p_0(z)-\sigma_d^{-2}A^\top(Az-y).
\]
Subtracting this identity from \eqref{eq:target-score} yields
\begin{equation}\label{eq:app-score-diff}
 \sbar_\lambda(z;y)-\nabla\log p(z\mid y)
 =\nabla\log\pi_\lambda(z)-\nabla\log p_0(z)
 -A^\top(\Gamma_\lambda^{-1}-\sigma_d^{-2}I)(Az-y).
\end{equation}
The first term converges locally uniformly by assumption. On every compact set, the second term converges uniformly because $\Gamma_\lambda^{-1}\to\sigma_d^{-2}I$ and $Az-y$ is bounded there.
\end{proof}

\section{Transport defect for the continued likelihood}\label{app:defect}
\begin{proof}[Proof of \cref{prop:explicit-defect}]
Write $f_\lambda=\pi_\lambda$ and
$u_\lambda=f_\lambda L_\lambda$, so that
$\pibar_\lambda=u_\lambda/\int u_\lambda$. Since
$\partial_\lambda f_\lambda=-q_\lambda\Delta f_\lambda$, centering removes
the derivative of the normalizer and gives
\begin{equation}\label{eq:app-centered-defect}
 r_\lambda
 =\mathcal C_{\pibar_\lambda}\!\left[
 \frac{\partial_\lambda u_\lambda+q_\lambda\Delta u_\lambda}
 {u_\lambda}\right].
\end{equation}
The product rule contributes
$2q_\lambda s_\lambda^{\pr}\cdot s_\lambda^{\lik}$. Moreover,
direct differentiation of \eqref{eq:Llambda} yields
\[
 \partial_\lambda\log L_\lambda
 =\frac12\dot\eta_\lambda\norm{s_\lambda^{\lik}}^2,
 \qquad
 \Delta\log L_\lambda=-\tr(H_\lambda).
\]
Consequently,
\[
 \frac{\partial_\lambda L_\lambda+q_\lambda\Delta L_\lambda}{L_\lambda}
 =-q_\lambda\tr(H_\lambda)
 +\left(q_\lambda+\frac12\dot\eta_\lambda\right)
 \norm{s_\lambda^{\lik}}^2.
\]
Substitution into \eqref{eq:app-centered-defect} proves
\eqref{eq:explicit-defect}, because the spatially constant trace term vanishes
under centering.
\end{proof}

\section{Entropy identity and finite-horizon tracking}\label{app:entropy}
\begin{proof}[Proof of \cref{thm:continuous}]
For notational brevity, write $q=q_\lambda$, $a=a_\lambda$, $\pi=\pibar_\lambda$, $\mu=\mu_\lambda$, and $g=\mu/\pi$. The Fokker--Planck equation of \eqref{eq:surrogate-sde} is
\begin{align}
 \partial_\lambda\mu
 &=-\nabla\cdot\bigl((q+a)\mu\nabla\log\pi\bigr)+a\Delta\mu\notag\\
 &=-q\nabla\cdot(\mu\nabla\log\pi)
 +a\nabla\cdot(\mu\nabla\log g).
 \label{eq:app-fp-relative}
\end{align}
The second equality follows from
$\nabla\cdot(\mu\nabla\log g)=\Delta\mu-\nabla\cdot(\mu\nabla\log\pi)$.
By the definition of the defect,
\begin{equation}\label{eq:app-target-evolution}
 \partial_\lambda\pi=-q\Delta\pi+r_\lambda\pi.
\end{equation}
Differentiate
$H(\lambda)=\int\mu\log g\dd z$. Since $\int\partial_\lambda\mu\dd z=0$,
\begin{equation}\label{eq:app-H-derivative}
 H'(\lambda)
 =\int\partial_\lambda\mu\log g\dd z
 -\int\mu\partial_\lambda\log\pi\dd z.
\end{equation}
The Langevin contribution in \eqref{eq:app-fp-relative} is
\begin{equation}\label{eq:app-mixing-term}
 a\int\nabla\cdot(\mu\nabla\log g)\log g\dd z
 =-a\int\norm{\nabla\log g}^2\dd\mu.
\end{equation}
The base transport contribution is
\begin{align}
 q\int\mu\nabla\log\pi\cdot\nabla\log g\dd z
 &=q\int\nabla\pi\cdot\nabla g\dd z\notag\\
 &=-q\int g\Delta\pi\dd z
 =-q\int\frac{\Delta\pi}{\pi}\dd\mu.
 \label{eq:app-transport-cancel}
\end{align}
Meanwhile, \eqref{eq:app-target-evolution} gives
\begin{equation}\label{eq:app-target-term}
 -\int\mu\partial_\lambda\log\pi\dd z
 =q\int\frac{\Delta\pi}{\pi}\dd\mu-\int r_\lambda\dd\mu.
\end{equation}
The first terms in \eqref{eq:app-transport-cancel} and \eqref{eq:app-target-term} cancel, yielding the exact identity \eqref{eq:entropy-identity}.

Apply the log-Sobolev inequality \eqref{eq:LSI} to $f=\sqrt g$. Since $\int g\dd\pi=1$,
\[
 H(\lambda)=\Ent_\pi(g),
 \qquad
 \int\norm{\nabla\sqrt g}^2\dd\pi
 =\frac14I(\mu\Vert\pi),
\]
and hence
\begin{equation}\label{eq:app-LSI-Fisher}
 I(\mu\Vert\pi)\ge2\rho_\lambda H(\lambda).
\end{equation}
Combining \eqref{eq:entropy-identity}, \eqref{eq:defect-control}, and \eqref{eq:app-LSI-Fisher},
\[
 H'(\lambda)
 \le-\bigl(2a_\lambda\rho_\lambda-c_\lambda\bigr)H(\lambda)+\varepsilon_\lambda
 =-k_\lambda H(\lambda)+\varepsilon_\lambda.
\]
The integrating-factor form of Gronwall's inequality gives
\eqref{eq:Gronwall}. If the cumulative contraction diverges and its weighted
defect forcing vanishes, then $H(\lambda)\to0$. Pinsker's inequality, the
triangle inequality, and \cref{prop:endpoint} then yield convergence to the
posterior in total variation.

For completeness, the entropy variational inequality
\[
 \int F\dd\nu\le\KL(\nu\Vert\pi)+\log\int e^F\dd\pi
\]
with $F=-\theta_\lambda r_\lambda$ yields
\eqref{eq:variational-defect} pointwise in $\lambda$.
\end{proof}

\section{Post-solve Langevin innovation and weak error}\label{app:weak}
\begin{proof}[Proof of \cref{prop:covariance}]
Write $R_k=(I+\Omega_kH_{k+1})^{-1}$ and
$\widetilde z_{k+1}=\widehat x_{0,k}+\varrho_k(z_k-\widehat x_{0,k})+
\Omega_kb_{k+1}$. Then
$z_{k+1}=R_k\widetilde z_{k+1}+\sqrt{V_k}\xi_k$, so its conditional
covariance is exactly $V_kI$. If the innovation is instead included in the
implicit solve, the update becomes
$R_k(\widetilde z_{k+1}+\sqrt{V_k}\xi_k)$ and has conditional covariance
$V_kR_k^2$, since $R_k$ is symmetric. This proves the covariance comparison.

On a fixed finite interval, Taylor expansion of the scalar integrals in
\eqref{eq:rho-step}--\eqref{eq:Omega-V} gives, uniformly in $k$,
\begin{equation}\label{eq:app-parameter-expansion}
 \varrho_k=1-\kappa_kh_k+O(h_k^2),\qquad
 \Omega_k=\chi_kh_k+O(h_k^2),\qquad
 V_k=2a_kh_k+O(h_k^2),
\end{equation}
because the exponential factor defining $V_k$ is $1+O(h_k)$ uniformly over
the step. This proves \eqref{eq:first-order-params}.

Tweedie's identity gives
$\widehat x_{0,k}-z=q_ks_k^{\pr}(z)$. Smoothness in $\lambda$ and
\eqref{eq:app-parameter-expansion} yield
\[
 \widetilde z_{k+1}
 =z+\chi_kh_k\bigl(s_k^{\pr}(z)+b_k\bigr)
 +O((1+\norm z)h_k^2),
 \qquad
 R_k=I-\chi_kh_kH_k+O(h_k^2).
\]
Consequently,
\begin{align}
 \E[z_{k+1}-z\mid z_k=z]
 &=(q_k+a_k)\sbar_{\lambda_k}(z;y)h_k
 +O((1+\norm z)h_k^2),\label{eq:app-mean-match}\\
 \operatorname{Cov}(z_{k+1}\mid z_k=z)
 &=2a_kh_kI+O(h_k^2),\label{eq:app-cov-match}
\end{align}
where the first line uses
$\sbar_{\lambda_k}(z;y)=s_k^{\pr}(z)-H_kz+b_k$. These are the first
conditional moments of \eqref{eq:surrogate-sde} to first order.
\end{proof}

\begin{proof}[Proof of \cref{thm:weak}]
Set
\[
 B_\lambda(z)=(q_\lambda+a_\lambda)\sbar_\lambda(z;y),
 \qquad
 \mathcal L_\lambda f
 =B_\lambda\cdot\nabla f+a_\lambda\Delta f.
\]
Fix $k$ and condition on $z_k^{\alg}=z$. The affine Gaussian update can be
written as
\[
 \Delta_k:=z_{k+1}^{\alg}-z=m_k(z)+\sqrt{V_k}\xi_k,
 \qquad \xi_k\sim\mathcal N(0,I),
\]
where \eqref{eq:app-mean-match}--\eqref{eq:app-cov-match} give
\begin{equation}\label{eq:app-first-moment}
 m_k(z)=B_{\lambda_k}(z)h_k+O((1+\norm{z})h_k^2),
 \qquad
 V_k=2a_{\lambda_k}h_k+O(h_k^2).
\end{equation}
The raw second moment, rather than only the covariance, is needed below. Since
$\E[\Delta_k\Delta_k^\top]=V_kI+m_km_k^\top$,
\begin{equation}\label{eq:app-second-moment}
 \E[\Delta_k\Delta_k^\top]
 =2a_{\lambda_k}h_kI
 +O((1+\norm{z}^2)h_k^2).
\end{equation}
The centered Gaussian third moment vanishes. Expanding around $m_k$ and using
\eqref{eq:app-first-moment} therefore gives, for some integer $m$ independent
of the grid,
\begin{equation}\label{eq:app-higher-moments}
 \left\|\E[\Delta_k^{\otimes3}]\right\|
 \le C(1+\norm{z}^m)h_k^2,
 \qquad
 \E\norm{\Delta_k}^4
 \le C(1+\norm{z}^m)h_k^2.
\end{equation}
The first estimate is a signed tensor-moment bound, whereas the second is the
absolute fourth-moment bound required for the Taylor remainder.

For a function $f$ whose spatial derivatives through order four have the stated
polynomial-growth bounds, a third-degree Taylor expansion with a fourth-order
remainder gives
\begin{align*}
 f(z+\Delta_k)
 ={}&f(z)+Df(z)[\Delta_k]
 +\frac12D^2f(z)[\Delta_k,\Delta_k]\\
 &+\frac16D^3f(z)[\Delta_k,\Delta_k,\Delta_k]
 +R_4(z,\Delta_k).
\end{align*}
The remainder satisfies
\[
 |R_4(z,\Delta_k)|
 \le C(1+\norm{z}^m+\norm{\Delta_k}^m)\norm{\Delta_k}^4.
\]
Gaussian moment estimates and \eqref{eq:app-first-moment} therefore bound its
conditional expectation by $C(1+\norm{z}^{m'})h_k^2$ for some $m'$. Taking
conditional expectations in the Taylor expansion and applying
\eqref{eq:app-first-moment}--\eqref{eq:app-higher-moments} yields
\begin{equation}\label{eq:app-numerical-local}
 \widehat Q_kf(z)=f(z)+h_k\mathcal L_{\lambda_k}f(z)
 +O((1+\norm{z}^m)h_k^2).
\end{equation}

We next derive the corresponding expansion for the exact transition. Set
$s=\lambda_k$ and $t=\lambda_{k+1}$, and write $\E_{s,z}$ for expectation
conditional on $Z_s=z$. Dynkin's formula gives
\[
 Q_kf(z)=f(z)+\int_s^t\E_{s,z}
 [\mathcal L_rf(Z_r)]\dd r.
\]
Applying the time-dependent Dynkin formula to
$g(r,x)=\mathcal L_rf(x)$ gives
\[
 \E_{s,z}[\mathcal L_rf(Z_r)]
 =\mathcal L_sf(z)
 +\int_s^r\E_{s,z}
 \left[((\partial_u+\mathcal L_u)\mathcal L_uf)(Z_u)\right]\dd u.
\]
Substitution into the preceding identity yields
\begin{align*}
 Q_kf(z)
 ={}&f(z)+h_k\mathcal L_sf(z)\\
 &+\int_s^t\int_s^r\E_{s,z}
 \left[((\partial_u+\mathcal L_u)\mathcal L_uf)(Z_u)\right]
 \dd u\dd r.
\end{align*}
The coefficient regularity implies that
$((\partial_u+\mathcal L_u)\mathcal L_uf)(x)$ has polynomial growth. The exact
moment bounds and the area $h_k^2/2$ of the integration region then give
\begin{equation}\label{eq:app-exact-local}
 Q_kf(z)=f(z)+h_k\mathcal L_{\lambda_k}f(z)
 +O((1+\norm{z}^m)h_k^2).
\end{equation}
Comparing \eqref{eq:app-numerical-local} and \eqref{eq:app-exact-local} gives
\begin{equation}\label{eq:app-local-difference}
 |(\widehat Q_k-Q_k)f(z)|\le C_f(1+\norm{z}^m)h_k^2.
\end{equation}

Finally, define the backward functions
\[
 u_k(z):=\E[\phi(Z_\Lambda)\mid Z_{\lambda_k}=z].
\]
They satisfy $u_k=Q_ku_{k+1}$ and, by assumption, have spatial derivatives
through order four with polynomial-growth bounds uniform in $k$. Since the
exact and numerical processes have the same initial law, telescoping gives
\begin{align*}
 \E\phi(z_K^{\alg})-\E\phi(Z_\Lambda)
 &=\sum_{k=0}^{K-1}
 \E\left[(\widehat Q_k-Q_k)u_{k+1}(z_k^{\alg})\right].
\end{align*}
Applying \eqref{eq:app-local-difference}, the numerical moment bounds, and
\[
 \sum_{k=0}^{K-1}h_k^2
 \le h\sum_{k=0}^{K-1}h_k
 =h(\Lambda-\lambda_0)
\]
proves \eqref{eq:weak-bound}.
\end{proof}

\section*{Acknowledgments and declarations}
The authors used OpenAI Codex to assist with language editing, \LaTeX{}
preparation, code development, and experimental workflow organization. The
authors assume responsibility for all content.

\end{document}